\pdfoutput=1
\documentclass{article}

\usepackage{graphicx,subfigure,natbib}
\usepackage{hyperref}

\usepackage[accepted]{icml2011} 

\usepackage{amsmath}
\usepackage{amssymb}

\usepackage{algorithm}
\usepackage{algorithmic}
\algsetup{indent=2em}

\usepackage{url}

\newtheorem{definition}{Definition}
\newtheorem{lemma}{Lemma}
\newtheorem{corollary}{Corollary}
\newtheorem{theorem}{Theorem}

\newcommand{\e}{{\mathbf e}}

\newcommand{\tr}{{\mathrm{tr}}}

\newcommand{\BlackBox}{\rule{1.5ex}{1.5ex}}  
\newenvironment{proof}{\par\noindent{\bf Proof\ }}{\hfill\BlackBox\\[2mm]}

\newcommand{\reals}{\mathbb{R}}

\newcommand{\V}{\mathcal{V}}
\newcommand{\U}{\mathcal{U}}
\newcommand{\inner}[1]{\langle #1 \rangle}

\newcommand{\rank}{\mathrm{rank}}
\newcommand{\supp}{\mathrm{supp}}
\newcommand{\spn}{\mathrm{span}}

\DeclareMathOperator*{\argmin}{argmin} 

\renewcommand{\eqref}[1]{Equation~(\ref{#1})}
\newcommand{\figref}[1]{Figure~\ref{#1}}
\newcommand{\secref}[1]{Section~\ref{#1}}
\newcommand{\thmref}[1]{Theorem~\ref{#1}}
\newcommand{\lemref}[1]{Lemma~\ref{#1}}

\newcommand{\blambda}{\bar{\lambda}}
\newcommand{\bA}{\bar{A}}
\newcommand{\bI}{\bar{I}}

\newcommand{\vect}{\mathrm{vec}}

\newcommand{\ShortPaper}[2]{#1}

\icmltitlerunning{Large-Scale Convex Minimization with a Low-Rank Constraint}

\begin{document}

\twocolumn[
\icmltitle{Large-Scale Convex Minimization with a Low-Rank Constraint}

\icmlauthor{Shai Shalev-Shwartz}{shais@cs.huji.ac.il}
\icmlauthor{Alon Gonen}{alongnn@gmail.com}
\icmladdress{School of Computer Science and Engineering, The Hebrew University of Jerusalem, ISRAEL}
\icmlauthor{Ohad Shamir}{ohadsh@microsoft.com}
\icmladdress{Microsoft Research New-England, USA}

\icmlkeywords{Rank Constraint, Trace-norm, Matrix Completion}

\vskip 0.3in
]

\begin{abstract}
We address the problem of minimizing a convex function over the space of large matrices with low rank. While this optimization problem is hard in general, we propose an efficient greedy algorithm and derive its formal approximation guarantees. Each iteration of the algorithm involves (approximately) finding the left and right singular vectors corresponding to the largest singular value of a certain matrix, which can be calculated in linear time. This leads to an algorithm which can scale to large matrices arising in several applications such as matrix completion for collaborative filtering and robust low rank matrix approximation.
\end{abstract}

\section{Introduction}

Our goal is to approximately solve an optimization problem of the form:
\begin{equation} \label{eqn:rankMin}
\min_{A : \rank(A) \le r} R(A) ~,
\end{equation}
where $R: \reals^{m \times n} \to \reals$ is a convex and smooth function. This problem arises in many machine learning applications such as collaborating filtering \cite{korenBeVo09},  robust low rank matrix approximation \cite{KeKan05,CrFil98,BacBeFal96}, and multiclass classification \cite{AmitFiSrUl07}.
The rank constraint on $A$ is non-convex and therefore it is generally NP-hard to solve \eqref{eqn:rankMin} (this follows from \cite{Natarajan95,DavisMaAv97}).

In this paper we describe and analyze an approximation algorithm for solving \eqref{eqn:rankMin}.
Roughly speaking, the proposed algorithm is based on a simple, yet powerful, observation: instead of representing a matrix $A$ using $m \times n$ numbers, we represent it using an infinite dimensional vector $\lambda$, indexed by all pairs $(u,v)$ taken from the unit spheres of $\reals^m$ and $\reals^n$ respectively. In this representation, low rank corresponds to sparsity of the vector $\lambda$.

Thus, we can reduce the problem given in \eqref{eqn:rankMin} to the problem of minimizing a vector function $f(\lambda)$ over the set of sparse vectors, $\|\lambda\|_0 \le r$. Based on this reduction, we apply a greedy approximation algorithm for minimizing a convex vector function subject to a sparsity constraint. At first glance, a direct application of this reduction seems impossible, since $\lambda$ is an infinite-dimensional vector, and at each iteration of the greedy algorithm one needs to search over the infinite set of the coordinates of $\lambda$. However, we show that this search problem can be cast as the problem of finding the first leading right and left singular vectors of a certain matrix.

After describing and analyzing the general algorithm, we show how to apply it  to the problems of matrix completion and robust low-rank matrix approximation. As a side benefit, our general analysis yields a new sample complexity bound for matrix completion. We demonstrate the efficacy of our algorithm by conducting experiments on large-scale movie recommendation data sets.

\subsection{Related work} \label{sec:related}

As mentioned earlier, the problem defined in \eqref{eqn:rankMin} has many applications, and therefore it was studied in various contexts. A popular approach is to use the trace norm as a surrogate for the rank (e.g. \cite{FazelHiBo02}). This approach is closely related to the idea of using the $\ell_1$ norm as a surrogate for sparsity, because low rank corresponds to sparsity of the vector of singular values and the trace norm is the $\ell_1$ norm of the vector of singular values. This approach has been extensively studied, mainly in the context  of collaborating filtering. See for example \cite{CaiCaSh08,CandesPl10,CandesRe09,KeshavanMoOh10,KeshavanOh09}.

While the trace norm encourages low rank solutions, it does not always produce sparse solutions.
Generalizing recent studies in compressed sensing, several papers (e.g. \cite{RechtFaPa07,CaiCaSh08,CandesPl10,CandesRe09,Recht09}) give recovery guarantees for the trace norm approach. However, these guarantees rely on rather strong assumptions (e.g., it is assumed that the data is indeed generated by a low rank matrix, that certain incoherence assumptions hold, and for matrix completion problems, it requires the entries to be sampled uniformly at random). In addition, trace norm minimization often involves semi-definite programming, which usually does not scale well to large-scale problems.

In this paper we tackle the rank minimization directly, using a greedy selection approach, without relying on the trace norm as a convex surrogate. Our approach is similar to forward greedy selection approaches for optimization with sparsity constraint (e.g. the MP \cite{MallatZa93} and OMP \cite{PatiReKr02} algorithms), and in particular we extend the fully corrective forward greedy selection algorithm given in \cite{ShalevSrZh10}). We also provide formal guarantees on the competitiveness of our algorithm relative to matrices with small trace norm.

Recently, \cite{LeeBr10} proposed the ADMiRA algorithm, which also follows the greedy approach.  However, the ADMiRA algorithm is different, as in each step it first chooses  $2r$ components and then uses SVD to revert back to a $r$ rank matrix. This is more expensive then our algorithm which chooses a single rank 1 matrix at each step. The difference between the two algorithms is somewhat similar to the difference between the OMP \cite{PatiReKr02} algorithm for learning sparse vectors, to CoSaMP \cite{NeedellTr09} and SP \cite{DaiMi08}. In addition, the ADMiRA algorithm is specific to the squared loss while our algorithm can handle any smooth loss. Finally, while ADMiRA comes with elegant performance guarantees, these rely on strong assumptions, e.g. that the matrix defining the quadratic loss satisfies a rank-restricted isometry property. In contrast, our analysis only assumes smoothness of the loss function.

The algorithm we propose is also related to Hazan's algorithm \cite{Hazan08} for solving PSD problems, which in turns relies on  Frank-Wolfe algorithm \cite{FrankWo56} (see  Clarkson \cite{Clarkson08}), as well as to the follow-up paper of \cite{JaggiSu10}, which applies Hazan's algorithm for optimizing with trace-norm constraints. There are several important changes though. First, we tackle the problem directly and do not enforce neither PSDness of the matrix nor a bounded trace-norm. Second, our algorithm is "fully corrective", that is, it extracts all the information from existing components before adding a new component. These differences between the approaches are analogous to the difference between Frank-Wolfe algorithm and fully corrective greedy selection, for minimizing over sparse vectors, as discussed in \cite{ShalevSrZh10}. Finally, while each iteration of both methods involves approximately finding leading eigenvectors, in \cite{Hazan08} the quality of approximation should improve as the algorithm progresses while our algorithm can always rely on the same constant approximation factor.

\section{The GECO algorithm}

In this section we describe our algorithm, which we call Greedy Efficient Component Optimization (or GECO for short).
Let $A \in \reals^{m \times n}$ be a matrix, and without loss of generality assume that $m \le n$. The SVD theorem states that $A$ can be written as
$A = \sum_{i=1}^m \lambda_i u_i v_i^T$,
where $u_1,\ldots,u_m$ are members of $\U = \{u  \in \reals^m : \|u\|=1\}$, $v_1,\ldots,v_m$ comes from $\V = \{v  \in \reals^n : \|v\|=1\}$, and $\lambda_1,\ldots,\lambda_m$ are scalars. To simplify the presentation, we assume that each real number is represented using a finite number of bits, therefore the sets $\U$ and $\V$ are finite sets.\footnote{This assumption greatly simplifies the presentation but is not very limiting since we do not impose any restriction on the amount of bits needed to represent a single real number. We note that the assumption is not necessary and can be waived by writing $A = \int_{(u,v)\in \U \times \V} u v^T ~d\lambda(u,v)$, where $\lambda$ is a measure on $\U \times \V$, and from the SVD theorem, there is always a representation with $\lambda$ which is non-zero on finitely many points.} It follows that we can also write $A$ as
$
A = \sum_{(u,v) \in \U \times \V} \lambda_{u,v} u v^T ~,
$
where $\lambda \in \reals^{|\U \times \V|}$ and we index the elements of $\lambda$ using pairs $(u,v) \in \U \times \V$. Note that the representation of $A$ using a vector $\lambda$ is not unique, but from the SVD theorem, there is always a representation of $A$ for which the number of non-zero elements of $\lambda$ is at most $m$, i.e. $\|\lambda\|_0 \le m$ where $\|\lambda\|_0 = |\{ (u,v) : \lambda_{u,v} \neq 0\}|$. Furthermore, if $\rank(A) \le r$ then there is a representation of $A$ using a vector $\lambda$ for which $\|\lambda\|_0 \le r$.

\begin{algorithm}[t]
\caption{GECO} \label{algo:GECO}
\begin{algorithmic}[1]
\STATE {\bf Input:} Convex-smooth function $R : \reals^{m \times n} \to \reals$   ~;~\\   rank constraint $r$ ~;~ tolerance $\tau \in [0,1/2]$
\STATE {\bf Initialize:} $U = []$, $V = []$
\FOR {i=1,\ldots,r}
\STATE $(u,v) = \mathrm{ApproxSV}(\nabla R(U V^T),\tau)$
\STATE Set $U = [U \,,\, u]$ and $V = [V \,,\, v]$
\STATE Set $B = \argmin_{B : \in \reals^{i \times i}} R(U B V^T)$
\STATE Calculate SVD:  $B = P D Q^T$
\STATE Update: $U = U P D$, $V = VQ$
\ENDFOR
\end{algorithmic}
\end{algorithm}

Given a (sparse) vector $\lambda \in \reals^{|\U \times \V|}$ we define the corresponding matrix to be
\[
A(\lambda) ~=~ \sum_{(u,v) \in \U \times \V} \lambda_{u,v} u v^T ~.
\]
Note that $A(\lambda)$ is a linear mapping.
Given a function $R : \reals^{m \times n} \to \reals$, we define a function
\[
f(\lambda) = R(A(\lambda)) = R\left(\sum_{(u,v) \in \U \times \V} \lambda_{u,v} u v^T \right) ~.
\]
It is easy to verify that if $R$ is a convex function over $\reals^{m \times n}$ then $f$ is convex over $\reals^{|\U \times \V|}$ (since $f$ is a composition of $R$ over a linear mapping).
We can therefore reduce the problem given in \eqref{eqn:rankMin} to the problem
\begin{equation} \label{eqn:sparseMin}
\min_{\lambda \in \reals^{|\U \times \V|}: \|\lambda\|_0  \le r} f(\lambda) ~.
\end{equation}

While the optimization problem given in \eqref{eqn:sparseMin} is over an arbitrary large space, we next show that a forward greedy selection procedure can be implemented efficiently.   The greedy algorithm starts with $\lambda = (0,\ldots,0)$. At each iteration, we first find the vectors $(u,v)$ that maximizes the magnitude of the partial derivative of $f(\lambda)$ with respect to $\lambda_{u,v}$.
Assuming that $R$ is differentiable, and using the chain rule, we obtain:
\begin{align*}
\frac{\partial f(\lambda)}{\partial \lambda_{u,v}} =
\inner{\nabla R(A(\lambda)), u v^T}
= u^T \nabla R(A(\lambda)) v ~,
\end{align*}
where $\nabla R(A(\lambda))$ is the $m \times n$ matrix of partial derivatives of $R$ with respect to the elements of $A(\lambda)$. The vectors $u,v$ that maximizes the magnitude of the above expression are the left and right singular vectors corresponding to the maximal singular value of $\nabla R(A(\lambda))$. Therefore, even though the number of elements in $\U \times \V$ is very large, we can still perform a greedy selection of one pair $(u,v) \in \U \times \V$ in an efficient way.

In some situations, even the calculation of the leading singular vectors might be too expensive. We therefore allow approximate maximization, and denote by $\mathrm{ApproxSV}(\nabla R(A(\lambda)),\tau)$ a procedure\footnote{\label{foot:power} An example of such a procedure is the power iteration method, which can implement ApproxSV in time $O(N \log(n)/\tau)$, where $N$ is the number of non-zero elements of $\nabla R(A(\lambda)) $. See Theorem 3.1 in \cite{KuczynskiWo92}.  Our analysis shows that the value of $\tau$ has a mild effect on the convergence of GECO, and one can even choose a constant value like $\tau = 1/2$. This is in contrast to \cite{Hazan08,JaggiSu10} which require the approximation parameter to decrease when the rank increases.  Note also that the ApproxEV procedure described in \cite{Hazan08,JaggiSu10} requires an additive approximation, while we require a multiplicative approximation. } which returns vectors for which
\[
u^T \nabla R(A(\lambda)) v \ge (1-\tau)\max_{p,q} p^T \nabla R(A(\lambda)) q  ~.
\]

Let $U$ and $V$ be matrices whose columns contain the vectors $u$ and $v$ we aggregated so far. The second step of each iteration of the algorithm sets $\lambda$ to be the solution of the following optimization problem:
\begin{equation} \label{eqn:optimize}
\min_{\lambda \in \reals^{|\U \times \V|}} f(\lambda) \\
~~\textrm{s.t.}~~ \supp(\lambda) \subseteq \spn(U) \times \spn(V),
\end{equation}
where $\supp(\lambda) = \{(u,v) : \lambda_{u,v} \neq 0\}$, and $\spn(U),\spn(V)$ are the linear spans of the columns of $U,V$ respectively.

We now describe how to solve \eqref{eqn:optimize}.  Let $s$ be the number of columns of $U$ and $V$. Note that any vector $u \in \spn(U)$ can be written as $U b_u$, where $b_u \in \reals^s$, and similarly, any $v \in \spn(V)$ can be written as $V b_v$. Therefore, if the support of $\lambda$ is in $\spn(U) \times \spn(V)$ we have that $A(\lambda)$ can be written as
\begin{align*}
A(\lambda) &= \sum_{(u,v) \in \supp(\lambda)} \lambda_{u,v} (U b_u) (V b_v)^T \\
&= U \left(\sum_{(u,v) \in \supp(\lambda)} \lambda_{u,v} b_u b_v^T \right) V^T .
\end{align*}
Thus, any $\lambda$ whose support is in $\spn(U) \times \spn(V)$ yields a matrix $B(\lambda) = \sum_{u,v} \lambda_{u,v} b_u b_v^T $. The SVD theorem tells us that the opposite direction is also true, namely, for any $B \in \reals^{s \times s}$ there exists $\lambda$ whose support is in $\spn(U) \times \spn(V)$ that generates $B$ (and also $UBV^T$).  Denote $\tilde{R}(B) = R(U B V^T)$, it follows that \eqref{eqn:optimize} is equivalent to the following unconstrained optimization problem
$\min_{B \in \reals^{s \times s}} \tilde{R}(B)$.
It is easy to verify that $\tilde{R}$ is a convex function, and therefore can be minimized efficiently. Once we obtain the matrix $B$ that minimizes $\tilde{R}(B)$ we can use its SVD to generate the corresponding $\lambda$.

In practice, we do not need to maintain $\lambda$ at all, but only to maintain matrices $U,V$ such that $A(\lambda) = UV^T$. A summary of the pseudo-code is given in Algorithm \ref{algo:GECO}. The runtime of the algorithm is as follows. Step 4 can be performed in time $O(N \log(n)/ \tau)$, where $N$ is the number of non zero elements of $\nabla R(U V^T)$, using the power method (see  Footnote \ref{foot:power}). Since our analysis (given in \secref{sec:genAnalysis}) allows $\tau$ to be a constant (e.g. $1/2$), this means that the runtime is $O(N \log(n))$. The runtime of Step 6 depends on the structure of the function $R$. We specify it when describing specific applications of GECO in later sections. Finally, the runtime of Step 7 is at most $r^3$, and step 8 takes $O(r^2(m+n))$.

\subsection{Variants of GECO} \label{sec:variants}

\subsubsection{How to choose $(u,v)$}   \label{chooseDirection}
GECO chooses $(u,v)$ to be the leading singular vectors, which are the maximizers of $u^T \nabla R(A)\, v$ over unit spheres of $\reals^m$ and $\reals^n$. Our analysis in the next section guarantees that this choice yields a sufficient  decrease of the objective function. However, there may be a pair $(u,v)$ which leads to an even larger decrease in the objective value. Choosing such a direction can lead to improved performance. We note that our analysis in the next section still holds, as long as the direction we choose leads to a larger decrease in the objective value, relative to the increase we can get from using the leading singular vectors. In \secref{sec:experiments} we describe a method that finds better directions.

\subsubsection{Additional replacement steps} \label{sec:replacements}
Each iteration of GECO increases the rank by $1$. In many cases, it is possible to decrease the objective by replacing one of the components without increasing the rank. If we verify that this replacement step indeed decreases the objective (by simply evaluating the objective before and after the change), then the analysis we present in the next section remains valid. We now describe a simple way to perform a replacement. We start with finding a candidate pair $(u,v)$ and perform steps $5-7$ of GECO. Then, we approximate the matrix $B$ by zeroing its smallest singular value. Let $\hat{B}$ denote this approximation. We next check if $R(U\hat{B}V^T)$ is strictly smaller than the previous objective value. If yes, we update $U,V$ based on $\hat{B}$ and obtain that the rank of $UV^T$ has not been increased while the objective has been decreased. Otherwise, we update $U,V$ based on $B$, thus increasing the rank, but our analysis tells us that  we are guaranteed to sufficiently decrease the objective. If we restrict the algorithm to perform at most $O(1)$ attempted replacement steps between each rank-increasing iteration, then its runtime guarantee is only increased by an $O(1)$ factor, and all the convergence guarantees remain valid.

\subsubsection{Adding Schatten norm regularization} \label{sec:reg}
In some situations, rank constraint is not enough for obtaining good generalization guarantees and one can consider objective functions $R(A)$ which contains additional regularization of the form $h(\lambda(A))$, where $\lambda(A)$ is the vector of singular values of $A$ and $h$ is a vector function such as $h(x) = \|x\|_p^2$. For example, if $p=2$, this regularization term is equivalent to Frobenius  norm regularization of $A$. In general, adding a convex regularization term should not pose any problem. A simple trick to do this is to orthonormalize the columns of $U$ and $V$ before Step 6. Therefore, for any $B$, the singular values of $B$ equal the singular values of $U B V^T$. Thus, we can solve the problem in Step 6 more efficiently while regularizing $B$ instead of the larger matrix $U B V^T$.

\subsubsection{Optimizing over diagonal matrices $B$}
Step $6$ of GECO involves solving a problem with $i^2$ variables, where $i \in \{1,\ldots,r\}$. When $r$ is small this is a reasonable computational effort. However, when $r$ is large, Steps $6-7$ can be expensive. For example, in matrix completion problems, the complexity of Step $6$ can scale with $r^6$. If runtime is important, it is possible to restrict $B$ to be a diagonal matrix, or in other words, we only optimize over the coefficients of $\lambda$ corresponding to $U$ and $V$ without changing the support of $\lambda$. Thus, in step $6$ we solve a problem with $i$ variables, and Step $7$ is not needed. It is possible to verify that the analysis we give in the next section still holds for this variant.

\section{Analysis} \label{sec:genAnalysis}

In this section we give a competitive analysis for GECO. The first theorem shows that after performing $r$ iterations of GECO, its solution is not much worse than the solution of \emph{all} matrices $\bar{A}$, whose trace norm\footnote{The trace norm of a matrix is the sum of its singular values.} is bounded by a function of $r$. The second theorem shows that with additional assumptions, we can be competitive with matrices whose rank is at most $r$.
\ShortPaper{The proofs can be found in the long version of this paper.}{The proofs are given in the Appendix.}

To formally state the theorems we first need to define a smoothness property of the function $f$.
\begin{definition}[smoothness]
We say that $f$ is $\beta$-smooth if for any $\lambda$ and $(u,v) \in \U \times \V$ we have
\[
f(\lambda + \eta \e^{u,v}) \le   f(\lambda) + \eta\, \frac{\partial f(\lambda)}{\partial \lambda_{u,v}}+  \frac{\beta\,\eta^2}{2}  ~,
\]
where $\e^{u,v}$ is the all zeros vector except $1$ in the coordinate corresponds to $(u,v)$.
We say that $R$ is $\beta$-smooth if the function $f(\lambda) = R(A(\lambda))$ is $\beta$-smooth.
\end{definition}

\begin{theorem} \label{thm:main}
Fix some $\epsilon > 0$. Assume that GECO (or one of its variants) is run with a $\beta$-smooth function $R$, a rank constraint $r$, and a tolerance parameter $\tau \in [0,1)$. Let $A$ be its output matrix. Then, for all matrices $\bar{A}$ with
\[ \|\bar{A}\|^2_{\tr} \le \frac{\epsilon\,(r+1) (1-\tau)^2}{2\beta} \]
we have that $R(A) \le R(\bar{A}) + \epsilon$.
\end{theorem}

The previous theorem shows competitiveness with matrices of low trace norm. Our second theorem shows that with additional assumptions on the function $f$ we can be competitive with matrices of low rank as well. We need the following definition.

\begin{definition}[strong convexity]
Let $I \subset \U \times \V$. We say that $f$ is $\sigma$-strongly-convex over $I$ if for any $\lambda_1,\lambda_2$ whose support\footnote{The support of $\lambda$ is the set of $(u,v)$ for which $\lambda_{u,v} \neq 0$.} is in $I$ we have
\[
f(\lambda_1)-f(\lambda_2) - \inner{\nabla f(\lambda_2),\lambda_1-\lambda_2} \ge
\frac{\sigma}{2} \|\lambda_1 - \lambda_2\|_2^2 ~.
\]
We say that $R$ is $\sigma$-strongly-convex over $I$ if the function $f(\lambda) = R(A(\lambda))$ is $\sigma$-strongly-convex over $I$.
\end{definition}

\begin{theorem} \label{thm:sec}
Assume that the conditions of \thmref{thm:main} hold.  Then, for any $\bar{A}$ such that
\[
\rank(\bA) \le \frac{\epsilon\,(r+1) (1-\tau)^2\,\sigma}{4\beta R(0)} ~.
\]
and such that $R$ is $\sigma$-strongly-convex over the singular vectors of $\bA$, we have that $R(A) \le R(\bar{A}) + \epsilon$.
\end{theorem}

We discuss the implications of these theorems for several applications in the next sections.

\section{Application I: Matrix Completion}

Matrix completion is the problem of predicting the entries of some unknown target matrix $Y \in \reals^{m \times n}$ based on a random subset of observed entries, $E \subset [m] \times [n]$. For example, in the famous Netflix problem, $m$ represents the number of users, $n$ represents the number of movies, and $Y_{i,j}$ is a rating user $i$ gives to movie $j$.  One approach for learning the matrix $Y$ is to find a matrix $A$ of low rank which approximately agrees with $Y$ on the entries of $E$ (in mean squared error terms). Using the notation of this paper, we would like to minimize the objective
\[
R(A) =  \frac{1}{|E|}\sum_{(i,j) \in E} (A_{i,j}-Y_{i,j})^2,
\]
over low rank matrices $A$.

We now specify GECO for this objective function. It is easy to verify that the $(i,j)$ element of $\nabla R(A)$ is $2(A_{i,j}-Y_{i,j})$ if $(i,j) \in E$ and $0$ otherwise.  The number of non-zero elements of $\nabla R(A)$ is at most $|E|$, and therefore Step 4 of GECO can be implemented using the power method in time $O(|E| \log(n))$.
Given matrices $U,V$, let $u_i$ be the $i$'th row of $U$ and $v_j$ be the $j$'th row of $V$. We have that the $(i,j)$ element of the matrix $U B V^T$ can be written as $\inner{\vect(u_i^T v_j) , \vect(B)}$, where $\vect$ of a matrix is the vector obtained by taking all the elements of the matrix column wise. We can therefore rewrite
$R(U B V^t)$ as $\frac{1}{|E|} \sum_{(i,j) \in E} (\inner{\vect(u_i^T v_j) , \vect(B)} - Y_{i,j})$, which makes Step 6 of GECO a vanilla least squares problem over at most $r^2$ variables. The runtime of this step is therefore bounded by $O(r^6 + |E|r^2)$.

\subsection{Analysis} \label{sec:analysis}

To apply our analysis for matrix completion we first bound the smoothness parameter.

\begin{lemma} \label{lem:smooth}
For matrix completion  the smoothness parameter is at most $2/|E|$.
\end{lemma}
\begin{proof}
For any $u,v$ and $i,j$ we can rewrite
$(A_{i,j} + \eta u_i v_j -Y_{i,j})^2$ as 
\[(A_{i,j}-Y_{i,j})^2
+ 2(A_{i,j}-Y_{i,j})\,\eta u_i v_j + \eta^2 u_i^2 v_j^2 ~.\]
Taking expectation over $(i,j) \in E$ we obtain:
\[
f(\lambda + \eta \e^{u,v}) \le f(\lambda) + \eta \nabla_{u,v}f(\lambda) + \eta^2 \frac{1}{|E|} \sum_{(i,j) \in E} u_i^2 v_j^2 ~.
\]
Since
$\sum_{(i,j) \in E} u_i^2 v_j^2 \le \sum_{i} u_i^2 \sum_j v_j^2 = 1$, the proof follows.
\end{proof}

Our general analysis therefore implies that for any $\bar{A}$, GECO can find a matrix with rank $r \le O(\|\bar{A}\|_\tr^2/(\epsilon |E|))$, such that $R(A) \le R(\bar{A}) + \epsilon$.

Let us now discuss the implications of this result for the number of observed entries required for predicting the entire entries of $Y$. Suppose that the entries $E$ are sampled i.i.d. from some unknown distribution $D \in \reals^{m \times n}$, $D_{i,j} \ge 0$ for all $i,j$ and $\sum_{i,j} D_{i,j} = 1$. Denote the generalization error of a matrix $A$ by
\[
F(A) = \sum_{i,j} D_{i,j} (A_{i,j} - Y_{i,j})^2 ~.
\]
Using generalization bounds for low rank matrices (e.g. \cite{SrebroAlJa05}), it is possible to show that for any matrix $A$ of rank at most $r$ we have that with high probability\footnote{To be more precise, this bound requires that the elements of $A$ are bounded by a constant. But, since we can assume that the elements of $Y$ are bounded by a constant, it is always possible to clip the elements of $A$ to the range of the elements of $Y$ without increasing $F(A)$.}
\[
|F(A)-R(A)| \le \tilde{O}(\sqrt{r(m+n)/|E|}) ~.
\]
Combining this with our analysis for GECO, and optimizing $\epsilon$, it is easy to derive the following:
\begin{corollary}
 Fix some matrix $\bA$. Then, GECO can find a matrix $A$ such that with high probability over the choice of the entries in $E$
 \[
 F(A) \le F(\bA) + \tilde{O}\left(  \left(\frac{\|\bar{A}\|_\tr \sqrt{m+n} }{|E|}  \right)^{2/3} \right) ~.
 \]
 \end{corollary}
 Without loss of generality assume that $m
\le n$. It follows that if $\|\bA\|_\tr$ is order of $\sqrt{mn}$ then order of $n^{3/2}$ entries are suffices to learn the matrix $Y$. This matches recent learning-theoretic guarantees for distribution-free learning with the trace norm \cite{ShalevSha11}.

\section{Application II: Robust Low Rank Matrix Approximation}

A very common problem in data analysis is finding a low-rank matrix $A$ which approximates a given matrix $Y$, namely solving $\min_{A:\text{rank}(A)\leq r} d(A,Y)$, where $d$ is some discrepancy measure. For simplicity, assume that $Y \in \reals^{n \times n}$. When $d(A,V)$ is the normalized Frobenius norm $d(A,V)=\tfrac{1}{n^2} \sum_{i,j}(A_{i,j}-Y_{i,j})^2$, this problem can be solved efficiently via SVD. However, due to the use of the Frobenius norm, this procedure is well-known to be sensitive to outliers.

One way to make the procedure more robust is to replace the Frobenius norm by a less sensitive norm, such as the $l_1$ norm $d(A,V)=\tfrac{1}{n^2} \sum_{i,j}|A_{i,j}-Y_{i,j}|$ (see for instance \cite{BacBeFal96},\cite{CrFil98},\cite{KeKan05}). Unfortunately, there are no known efficient algorithms to obtain the global optimum of this objective function, subject to a rank constraint on $A$. However, using our proposed algorithm, we can efficiently find a low-rank matrix which approximately minimizes $d(A,V)$. In particular, we can apply it to any convex discrepancy measure $d$, including robust ones such as the $l_1$ norm. The only technicality is that our algorithm requires $d$ to be smooth, which is not true in the case of the $l_1$ norm. However, this can be easily alleviated by working with smoothed versions of the $l_1$ norm, which replace the absolute value by a smooth approximation. One example is a Huber loss, defined as $L(x)=x^2/2$ for $|x|\leq 1$, and $L(x)=|x|-1/2$ otherwise.

\begin{lemma} \label{lem:smoothRPCA}
The smoothness parameter of $d(A,Y)=\frac{1}{n^2}\sum_{i,j}L(A_{i,j}-Y_{i,j})$, where $L$ is the Huber loss, is at most $1/n^2$.
\end{lemma}
\begin{proof} 
It is easy to verify that the smoothness parameter of $L(x)$ is $1$, since $L(x)$ is upper bounded by the parabola $x^2/2$, whose smoothness parameter is exactly $1$. Therefore,
\begin{align*}
L(A_{i,j}+\eta u_i v_j-Y_{i,j}) &\leq L(A_{i,j}-Y_{i,j})\\
&+\eta L'(A_{i,j}-Y_{i,j})u_i v_j + \frac{\eta^2}{2}u_i^2 v_j^2.
\end{align*}
Taking the average over all entries, this implies that
\[
f(\lambda+\eta \mathbf{e}^{u,v})\leq f(\lambda)+\eta \nabla_{u,v}f(\lambda)+\frac{\eta^2}{n^2}\sum_{i,j}u_i^2 v_j^2.
\]
Since the last term is at most $\eta^2/n^2$, the result follows.
\end{proof}
%
We therefore obtain:
\begin{corollary}
Let $d(A,Y)$ be the Huber loss discrepancy as defined in \lemref{lem:smoothRPCA}. Then, for any matrix $\bA$, GECO can find a matrix $A$ with $d(A,Y) \le d(\bA,Y)+\epsilon$ and  $\rank(A) = O(\frac{\|\bA\|_\tr^2}{n^2\epsilon})$.
\end{corollary}

\begin{figure*}
\begin{center}
\includegraphics[width=0.65\columnwidth]{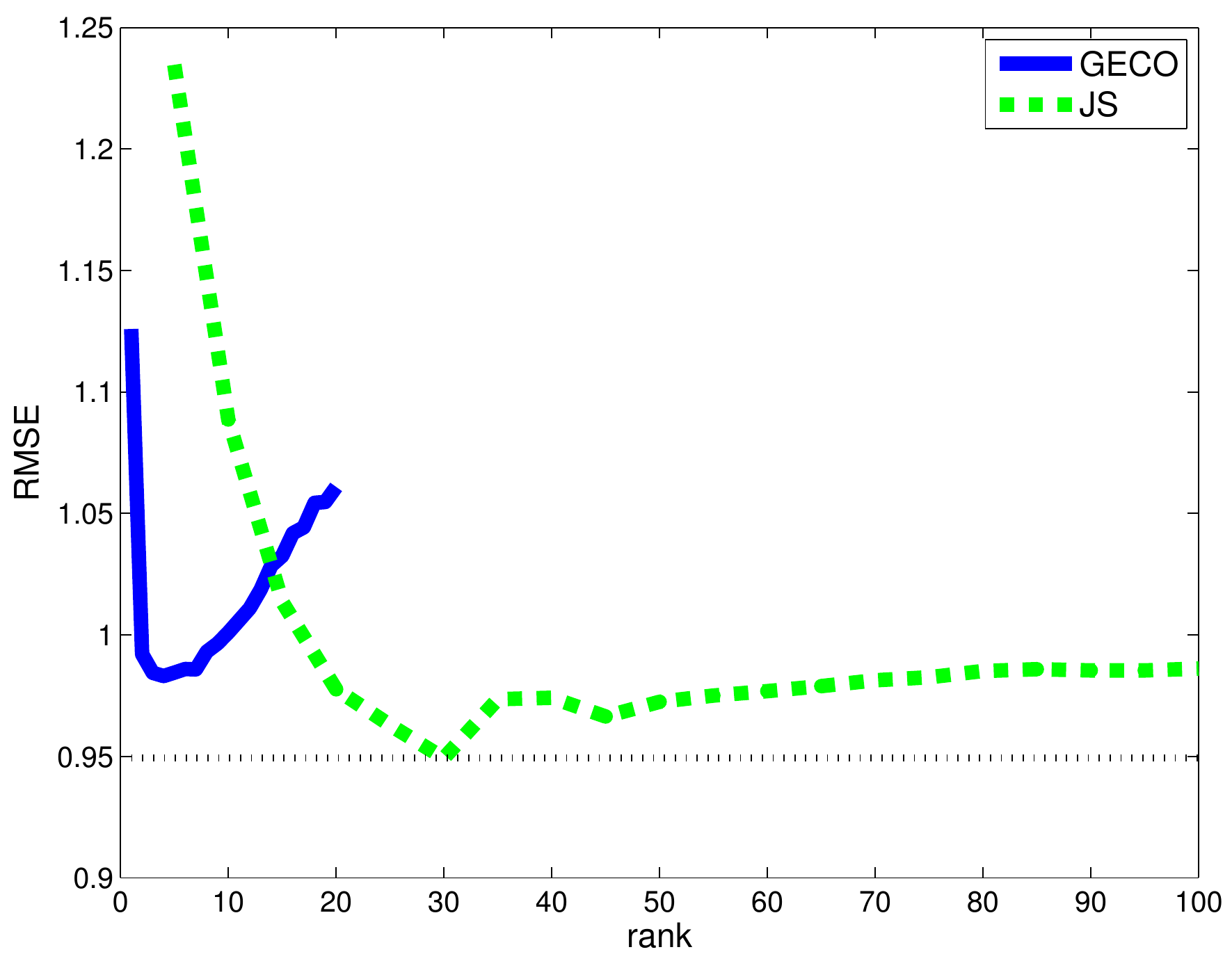}
\includegraphics[width=0.65\columnwidth]{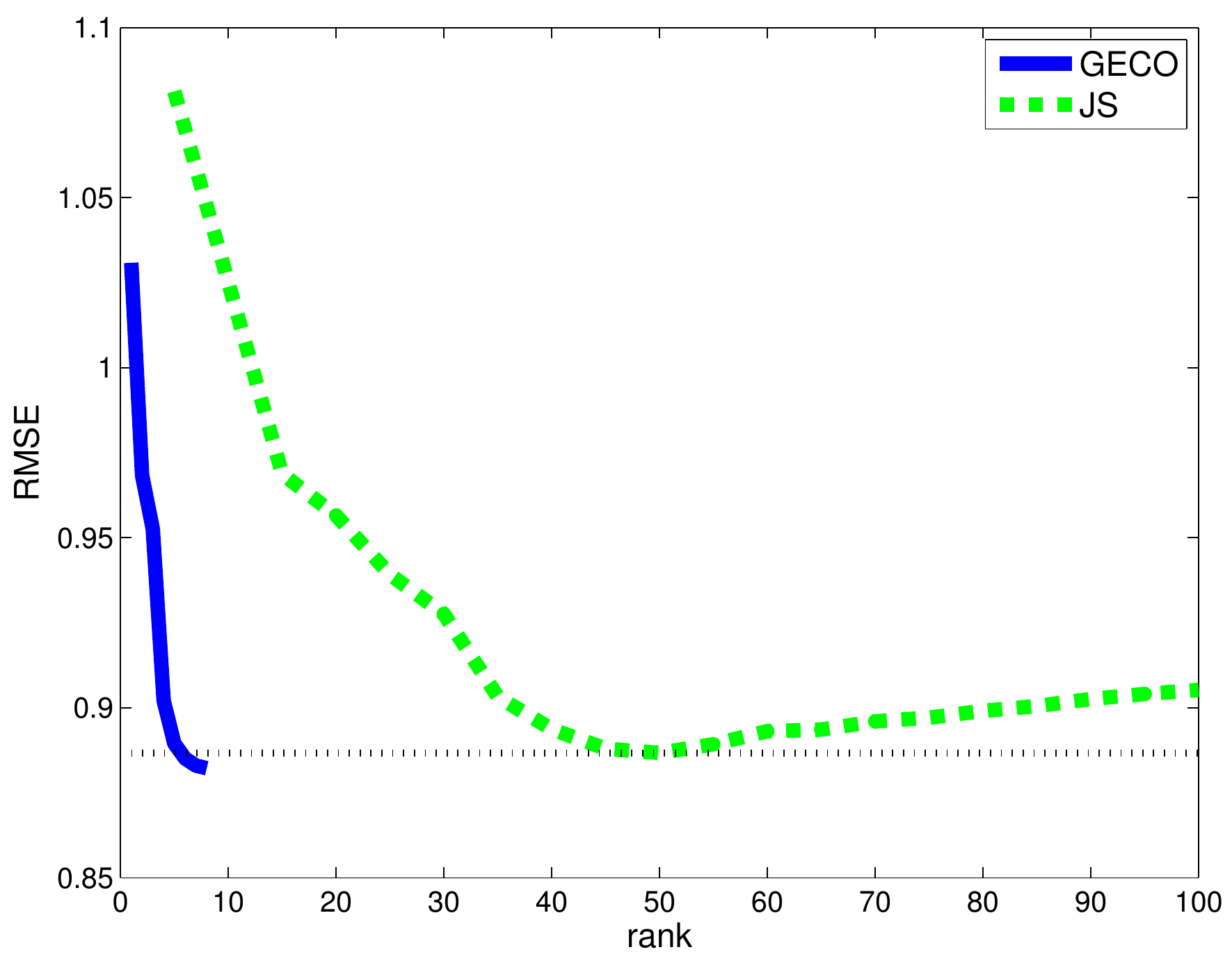}
\includegraphics[width=0.65\columnwidth]{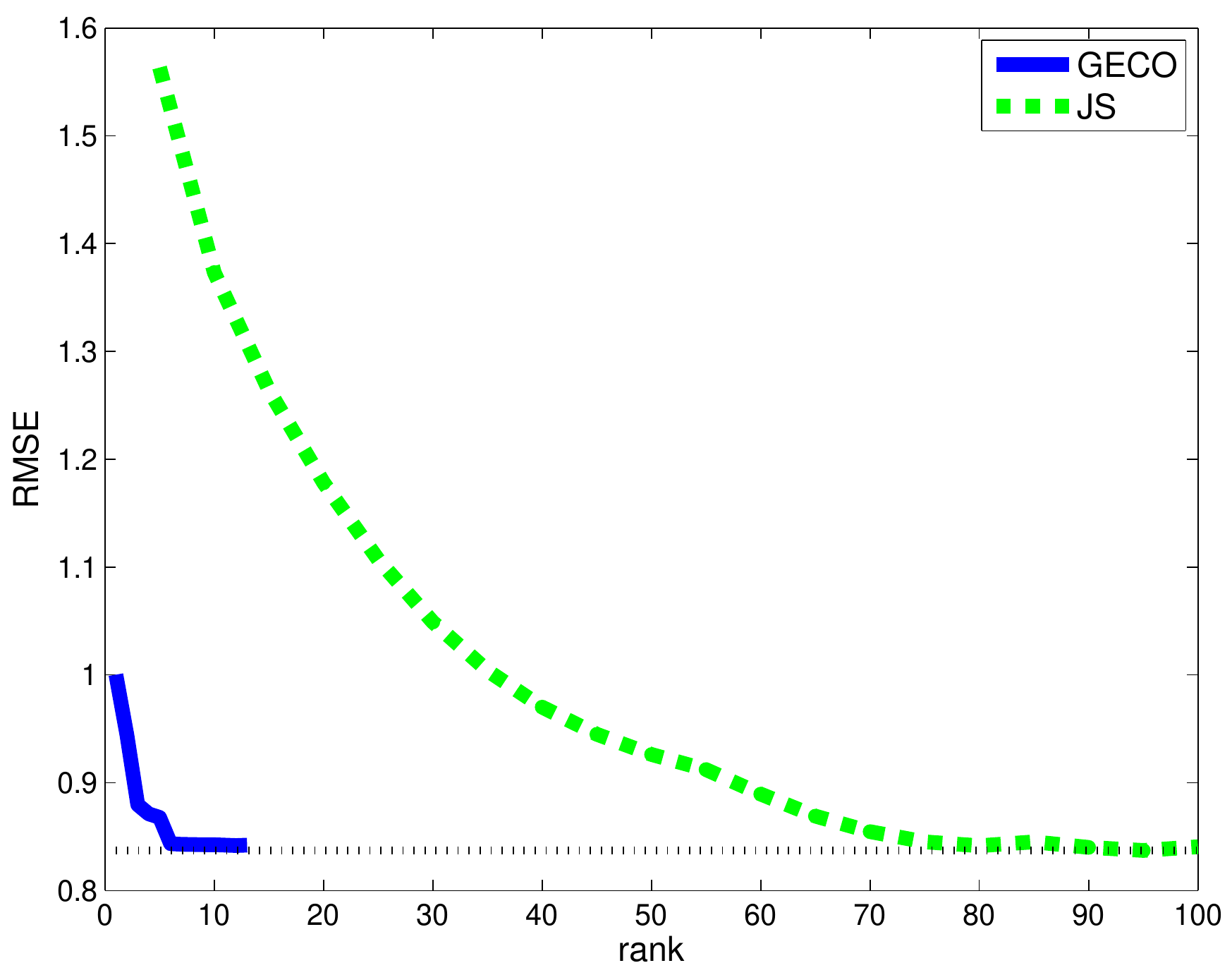}
\end{center}
\caption{Root Mean Squared Error on the test set as a function of the rank.  The horizontal line corresponds to the minimal error achieved by JS.  Left: MovieLens100k, Middle: MovieLens1M, Right: MovieLens10M.}
\label{fig:results}
\end{figure*}

\section{Experiments} \label{sec:experiments}

We evaluated GECO for the problem of matrix completion by conducting experiments on three standard collaborative filtering datasets: MovieLens100K, MovieLens1M, and MovieLens10M\footnote{Available through \url{www.grouplens.org}}. The different datasets contain $10^5,10^6,10^7$ ratings of $943,6040,69878$ users on $1682, 3706, 10677$ movies, respectively. All the ranking are integers in $1-5$. We partitioned each data set into training and testing sets as done in \cite{JaggiSu10}.

We implemented GECO while applying two of the variants described in \secref{sec:variants} as we explain in details below. The first variant (see \secref{chooseDirection}) tries to find update vectors $(u',v')$ which leads to a larger decrease of the objective function relatively to the leading singular vectors $(u,v)$ of the gradient matrix $\nabla R(A)$.  Inspired by the proof of \thmref{thm:main}, we observe that the decrease of the objective function inversely depends on the smoothness of the scalar function $R(A + \eta u v^t)$. We therefore would like to find a pair which on one hand has a large correlation with $\nabla R(A)$ and on the other hand yields a smooth scalar function $R(A + \eta u v^t)$.  The smoothness of $R(A + \eta u v^t)$ is analyzed in \lemref{lem:smooth} and is shown to be at most $\frac{2}{|E|}$. Examining the proof lines more carefully, we see that for balanced vectors, i.e. $u_i = \pm \frac{1}{\sqrt{m}}, v_j = \pm \frac{1}{\sqrt{n}}$, we obtain a lower smoothness parameter of $\frac{2}{mn}$. Thus, a possible good update direction is to choose $u,v$ that maximizes $u^T \nabla R(A) v$ over vectors of the form $u_i = \pm \frac{1}{\sqrt{m}}, v_j = \pm \frac{1}{\sqrt{n}}$. This is equivalent to maximizing $u^T \nabla R(A) v$ over the $\ell_\infty$ balls of $\reals^m$ and $\reals^n$, which is unfortunately known to be NP-hard. Nevertheless, a simple alternate maximization approach is easy to implement and often works well. That is, fixing some $u$, we can see that $v = \mathrm{sign}(u^T \nabla(A))/\sqrt{n}$ maximizes the objective, and similarly, fixing $v$ we have that $u = \mathrm{sign}(\nabla R(A) v)/\sqrt{m}$ is optimal. We therefore implement this alternate maximization at each step and find a candidate pair $(u',v')$. As described in section \secref{chooseDirection}, we compare the decrease of loss as obtained by the leading singular vectors, $(u,v)$, and the candidate pair mentioned previously, $(u',v')$, and update using the pair which leads to a larger decrease of the objective. We remind the reader that  although $(u',v')$ are obtained heuristically, our implementation is still provably correct and our guarantees from \secref{sec:genAnalysis} still hold.

In addition we performed the additional replacement steps as described in \secref{sec:replacements}. For that purpose, let $q$ be the number of times we try to perform additional replacement steps for each rank. Each  replacement attempt is done using the alternate maximization procedure described previously. After utilizing $q$ attempts of additional replacement steps, we force an increase of the rank. In our experiments, we set $q=20$. 
Finally, we implemented the ApproxSV procedure using $30$ iterations of the power iteration method. 


We compared GECO to a state-of-the-art method, recently proposed in \cite{JaggiSu10}, which we denote as the JS algorithm.  JS, similarly to GECO, iteratively increases the rank by computing a direction that maximizes some objective function and performing a step in that direction. See more details in \secref{sec:related}. In \figref{fig:results}, we plot the root mean squared error (RMSE) on the test set as a function of the rank. As can be seen, GECO decreases the error much faster than the JS algorithm. This is expected --- see again the discussion in \secref{sec:related}. We observe that GECO achieves slightly larger test error on the small data set, slightly smaller test error on the medium data set, and the same error on the large data set. On the small data set, GECO starts to overfit when the rank increases beyond $4$. The JS algorithm avoids this overfitting by constraining the trace-norm, but also starts overfitting after around $30$ iterations. On the other hand, on the medium data, the trace-norm constraint employed by the JS algorithm yields a higher estimation error, and GECO, which does not constrain the trace-norm, achieves a smaller error.  In any case, GECO achieves very good results while using a rank of at most $10$.

\section{Discussion}

GECO is an efficient greedy approach for minimizing a convex function subject to a rank constraint. 
One of the main advantages of GECO is that each of its iterations involves running few (precisely, $O(\log(n))$) iterations of the power method, and therefore GECO scales to large matrices. In future work we intend to apply GECO to additional applications such as multiclass classification and learning fast quadratic classifiers. 

\subsection*{Acknowledgements}
This work emerged from fruitful discussions with Tomer Baba, Barak Cohen, Harel Livyatan, and Oded Schwarz. The work is supported by the Israeli Science Foundation grant number 598-10.

{\small
\bibliography{curRefs} 
\bibliographystyle{icml2011}
}

%
%

\ShortPaper{

\newpage
\appendix

\section{Proofs}

\subsection{Proof of \thmref{thm:main}}

To prove the theorem we need the following key lemma, which generalizes a result given in \cite{ShalevSrZh10}.

\begin{lemma} \label{lma:key}
Assume that $f$ is $\beta$-smooth.
Let $I,\bar{I}$ be two subsets of $\U \times \V$. Let $\lambda$ be a minimizer of $f(\lambda)$ over all vectors with support in $I$ and let $\bar{\lambda}$ be a vector supported on $\bar{I}$. Assume that  $f(\lambda) > f(\bar{\lambda})$, denote $s=\|\blambda\|_1$, and let 
$\tau \in [0,1)$.
Let $(u,v) = \mathrm{ApproxSV}(\nabla R(A(\lambda)),\epsilon)$.
Then, there exists $\eta$ such that
\[
f(\lambda) - f(\lambda + \eta \e^{u,v})  \ge \frac{(f(\lambda)-f(\blambda))^2(1-\tau)^2}{2\beta s^2} ~.
\]
\end{lemma}
\begin{proof} 

Without loss of generality assume that $\bar{\lambda} \ge 0$ (if $\bar{\lambda}_{p,q} < 0$ for some $(p,q)$ we can set $\bar{\lambda}_{-p,q}=-\bar{\lambda}_{p,q}$ and $\bar{\lambda}_{p,q}=0$ without effecting the objective) and assume that $u^t \nabla(R(A(\lambda))) v \le 0$ (if this does not hold, let $u=-u$). For any $(p,q)$, let $\nabla_{p,q} = p^t \nabla R(A(\lambda)) q$ be the partial derivative of $f$ w.r.t. coordinate $(p,q)$ at $\lambda$ and denote
\[
Q_{p,q}(\eta) = f(\lambda) + \eta \, \nabla_{p,q}  + \frac{\beta\,\eta^2}{2}.
\]
 Note that the definition of $(u,v)$ and our assumption above implies  that
\[
-\nabla_{u,v}  = |\nabla_{u,v}| \ge (1-\tau) \max_{p,q} |\nabla_{p,q}| ~,
\]
which gives
\[
\nabla_{u,v} \le (\tau-1) \max_{p,q} |\nabla_{p,q}| = (1-\tau) \min_{p,q} \nabla_{p,q} ~.
\]
 Therefore, for all $\eta \ge 0$ we have
\[
 Q_{u,v}(\eta)  \le f(\lambda) +  (1-\tau) \eta \min_{p,q} \nabla_{p,q} + \frac{\beta \eta^2}{2} ~.
\]
In addition, the smoothness  assumption tells us that for all $\eta$ we have $f(\lambda + \eta \e^{u,v}) \le Q_{u,v}(\eta)$. Thus,  for any $\eta \ge 0$ we have
\begin{align*}
\min_{a} f(\lambda + a \e^{u,v}) &\le f(\lambda + \eta \e^{u,v}) \le Q_{u,v}(\eta)
\end{align*}
Combining the above we get
\[
\min_{a} f(\lambda + a \e^{u,v}) \le f(\lambda) +  (1-\tau) \eta \min_{(p,q) \in \bar{I}\setminus I} \nabla_{p,q} + \frac{\beta \eta^2}{2} ~.
\]
Multiplying both sides by $s$ and noting that
\begin{align*}
s \min_{(p,q) \in \bar{I}\setminus I} \nabla_{p,q} &\le \sum_{(p,q) \in \bar{I} \setminus I} \bar{\lambda}_{p,q} \nabla_{p,q}
\end{align*}
we get that
\begin{align*}
s& \min_{a} f(\lambda + a \e^{u,v})  \\
  &\le s f(\lambda) + (1-\tau) \eta \sum_{(p,q) \in \bar{I} \setminus I} \bar{\lambda}_{p,q} \nabla_{p,q} + s  \frac{\beta\,\eta^2}{2} ~.
\end{align*}
Since $\lambda$ is a minimizer of $f$ over $I$ we have that $ \nabla_{p,q}  = 0$ for $(p,q) \in I$. Combining this with the fact that $\lambda$ is supported on $I$ and $\blambda$ is supported on $\bI$ we obtain that
\[
 \sum_{(p,q) \in \bar{I} \setminus I} \bar{\lambda}_{p,q} \nabla_{p,q} = \inner{\bar{\lambda},\nabla f( \lambda)} = \inner{\bar{\lambda}-\lambda,\nabla f( \lambda)}~.
\]
From the convexity of $f$ we know that $\inner{\bar{\lambda}-\lambda,\nabla f( \lambda)} \le f(\bar{\lambda}) - f(\lambda)$. Combining all the above we obtain
\[
 s \min_{a} f(\lambda + a \e^{u,v}) \le
 s f(\lambda) + (1-\tau) \eta (f(\bar{\lambda})-f(\lambda))  + s  \frac{\beta\,\eta^2}{2} ~.
\]
This holds for all $\eta \ge 0$ and in particular for $\eta = (f(\lambda)-f(\blambda))(1-\tau)/(s\beta)$ (which is positive). Thus,
\[
s \min_{a} f(\lambda + a \e^{u,v}) \le  s f(\lambda) - \frac{(f(\lambda)-f(\blambda))^2(1-\tau)^2}{2\beta s} ~.
\]
Rearranging the above concludes our proof.
\end{proof}

Equipped with the above lemma we are ready to prove \thmref{thm:main}.

Fix some $\bA$ and let $\blambda$ be the vector of its singular values. Thus, $\|\blambda\|_1 = \|\bA\|_{\tr}$ and $f(\blambda) = R(\bA)$. For each iteration $i$, denote $\epsilon_i = f(\lambda^{(i)}) - f(\bar{\lambda})$, where $\lambda^{(i)}$ is the value of $\lambda$ at the beginning of iteration $i$ of GECO, before we increase the rank to be $i$. Note that all the operations we perform in GECO or one if its variants guarantee that the loss is monotonically non-increasing. Therefore, if $\epsilon_i \le \epsilon $ we are done. In addition, whenever we increase the rank by $1$, the definition of the update implies that $f(\lambda^{(i+1)}) \leq \min_{\eta} \, f(\lambda^{(i)} + \eta \e^{u, v})$, where $(u,v) = \mathrm{ApproxSV}(R(A(\lambda^{(i)})),\tau)$.
\lemref{lma:key} implies that 
\begin{equation} \label{eqn:fully-onestep-strong}
\begin{split}
\epsilon_i - \epsilon_{i+1} &= f(\lambda^{(i)}) - f(\lambda^{(i+1)})
 \geq~
\frac{\epsilon_i^2 (1-\tau)^2}{2\,\beta\,
\|\bA\|_{\tr}^2}  ~.
\end{split}
\end{equation}
Using Lemma B.2 from \cite{ShalevSrZh10}, the above implies that
for $i \ge 2\,\beta\,\|\bar{A}\|_{\tr}^2/(\epsilon(1-\tau)^2)$ we have that $\epsilon_i \le \epsilon$. We obtain that if $\|\bA\|^2_\tr \le \epsilon\,(r+1) (1-\tau)^2/ (2\beta)$ then $\epsilon_{r+1} \le \epsilon$, which concludes the proof of \thmref{thm:main}. \hfill\BlackBox

\subsection{Proof of \thmref{thm:sec}}
Let $\blambda$ be the vector obtained from the SVD of $\bA$, that is, $\bA = A(\blambda)$ and $\|\blambda\|_0 = \rank(\bA)$. Note that $f$ is $\sigma$-strongly-convex over the support of $\blambda$. Using Lemma 2.2 of \cite{ShalevSrZh10} we know that
$
\|\blambda\|_1^2 \le \frac{2 \|\blambda\|_0 \,f(0)}{\sigma}
$.
But, since $\|\bA\|_\tr = \|\blambda\|_1$, $\rank(\bA) = \|\blambda\|_0$, and $f(0) = R(0)$, we get
\[
\|\bA\|_\tr^2 \le \frac{2 \rank(\bA) \,R(0)}{\sigma} ~.
\]
The proof follows from the above using \thmref{thm:main}.
\hfill\BlackBox

}{}

\end{document}